\documentclass[twoside,11pt]{article}
\usepackage{jair, rawfonts}
\usepackage{apalike}

\usepackage{booktabs}       % professional-quality tables
\usepackage{amsfonts}       % blackboard math symbols
\usepackage{nicefrac}       % compact symbols for 1/2, etc.
\usepackage{microtype}      % microtypography
\usepackage{xcolor}         % colors
\usepackage{enumitem}
\usepackage{breakcites}

\usepackage{mathtools}
\usepackage{amsmath}
\usepackage{amsthm}
\usepackage{amsfonts}
\usepackage{amssymb}

\usepackage{etoolbox}
\usepackage{tcolorbox}
\newtcolorbox{axiombox}[2][]{colback=red!5!white,colframe=red!60!black,boxsep=2pt,left=6pt,right=6pt,top=4pt,bottom=4pt,title={\textbf{Axiom:} #2},#1}
\newtheorem{theorem}{Theorem}
\newtheorem{lemma}{Lemma}[section]
\newtheorem{definition}[lemma]{Definition}
\newtheorem{corollary}[lemma]{Corollary}
\newtheorem{remark}[lemma]{Remark}
\newtheorem{example}[lemma]{Example}
\newtheorem*{theorem*}{Theorem}
\newtheorem*{corollary*}{Corollary}

\usepackage{xcolor} % Ensure xcolor is loaded for \changed
\newcommand{\changed}[1]{{#1}}

\ShortHeadings{Representative Social Choice: From Learning Theory to AI Alignment}
{Qiu}
\firstpageno{1}

\begin{document}

\title{Representative Social Choice:\\From Learning Theory to AI Alignment}

\author{\name Tianyi Qiu \email qiutianyi.qty@gmail.com \\[4pt]
\addr Peking University \\
\addr No.5 Yiheyuan Rd, Beijing 100871 \\[4pt]
\addr Center for Human-Compatible AI, UC Berkeley \\
\addr 2121 Berkeley Way, CA 94720}

% For research notes, remove the comment character in the line below.
% \researchnote

\maketitle

\begin{abstract}
Social choice theory is the study of preference aggregation across a population, used both in mechanism design for human agents and in the democratic alignment of language models.
In this study, we propose the \emph{representative social choice} framework for the modeling of democratic representation in collective decisions, where the number of issues and individuals are too large for mechanisms to consider all preferences directly. These scenarios are widespread in real-world decision-making processes, such as jury trials, legislation, corporate governance, and, more recently, language model alignment.
In representative social choice, the population is \emph{represented} by a finite sample of individual-issue pairs based on which social choice decisions are made. We show that many of the deepest questions in representative social choice can be formulated as statistical learning problems, and prove the generalization properties of social choice mechanisms using the theory of machine learning. We further formulate axioms for representative social choice, and prove Arrow-like impossibility theorems with new combinatorial tools of analysis. Our framework introduces the representative approach to social choice, opening up research directions at the intersection of social choice, learning theory, and AI alignment.
\end{abstract}

\section{Introduction}

Social choice theory is a field of study that deals with the aggregation of individual preferences to form a collective decision. It has been applied in domains such as economics \cite{feldman2006welfare}, political science \cite{miller1983pluralism,coleman1986democracy}, and computer science \cite{conitzerposition}, to name a few. In these applications, the goal is to design mechanisms that aggregate individual preferences in a way that satisfies certain desirable properties, most especially fairness.

However, existing models in social choice theory tend to be simplistic and rely on relatively strong assumptions, including (1) \emph{independent, single-issue choices}, and (2) \emph{complete information on all preferences of all individuals}. In practice, these assumptions are often violated. In common large-scale elections, candidates have partisan policies that are correlated across a huge number of different issues, and it is also infeasible to collect preferences of all voters on all issues, due to the large number of issues and voters involved.

These problems are not merely practical details that can be ignored. They are fundamental to the theory of social choice itself, since these complexities are exactly what give rise to democratic \emph{representation} --- the idea that individuals can delegate their decision-making power to a small number of representatives who can make decisions on their behalf, when there are too many issues and too many individuals to consider all preferences directly. The introduction of representation leads to fundamental questions for social choice theory that are not well-understood, such as the problem of \emph{generalization} --- how can we ensure that the decisions made by the representatives are representative of the population's preferences, when the representatives are only chosen based on a small number of individuals' opinions on a small number of issues?

In this paper, we propose a new framework for social choice theory that models these complexities, which we call \emph{representative social choice}. As we will show in the following sections, many of the deepest questions in representative social choice can be formulated as statistical learning problems, despite the seemingly different natures of the two fields. This connection allows us to leverage the rich theory of statistical machine learning to formulate axioms and mechanisms for representative social choice, and to analyze their properties. \changed{While some work studies winner determination for a single election from voter samples \cite{bhattacharyya2021predicting}, our framework addresses the aggregation of preferences over a profile of many issues, making generalization over both individuals and issues a central challenge.}

\paragraph{Applications} The representative model can be applied in the modeling of all \emph{collective decisions involving representation and delegation}, which include most real-world decision-making processes, such as:
\vspace{-0.5em}
\begin{itemize}[leftmargin=2em]
  \setlength\itemsep{0.1em}
	\item \emph{Jury Trials}: Citizens delegate their decision-making power to a randomly selected jury.
	\item \emph{Legislation Processes}: Legislature body members, which can be viewed as representative samples of the citizen population, delegates the population's legal decicions to a set of laws that they decide on. We have a multi-issue setting, since the laws cover a wide range of issues. The space of possible collective preference profiles is the space of profiles implementable by laws.
	\item \emph{Corporate Governance}: Shareholders elect a board of directors to make decisions on their behalf. We have a multi-issue setting, since the board decides on a wide range of issues during its service.
	\item \emph{AI Alignment}: Frontier AI systems, including large language models (LLMs), undergo the \emph{alignment} process during training, where they are trained to make decisions that are aligned with human values \cite{bai2022training,bai2022constitutional}, using preference datasets sampled from humans. Alignment can be viewed as a form of representative social choice, where the LLM is trained to make decisions that are representative of the human population's preferences, the latter represented by individual-issue pairs sampled from human evaluators (\emph{i.e.}, the preference dataset). The space of possible collective preference profiles is the space of profiles that can be actualized as LLM policies --- a \emph{feature space} that has been the subject of much research in the field of statistical learning theory \cite{vapnik1999overview}.
\end{itemize}
\vspace{-0.5em}

\section{Related Work}\label{sec:related_work}

Social choice theory has had a long history \cite{satterthwaite1975strategy,young1975social,nisan1999algorithmic}, with more recent research studying its intersection with machine learning \cite{fish2023generative,parkes2013dynamic}, and its applications in AI alignment \cite{conitzerposition,kopf2024openassistant,klingefjord2024human,huang2024collective,prasad2018social,mishra2023ai,ge2024axioms}. This paper contributes to all three fronts, by extending the study of social choice theory to include democratic representation, using tools from the theory of machine learning, and with important applications in both human society and AI alignment.

\paragraph{Social Choice Theory} Social choice theory studies the aggregation of individual preferences to form a collective decision. The founding work of the field \cite{arrow2012social} proved the famous \emph{Arrow's impossibility theorem}, stating that no social choice mechanism can satisfy the axioms of \emph{unrestricted domain}, \emph{non-dictatorship}, \emph{Pareto efficiency}, and \emph{independence of irrelevant alternatives}. Many extensions of Arrow's theorem have been proposed, including the Gibbard-Satterthwaite theorem \cite{gibbard1973manipulation,satterthwaite1975strategy} which introduces strategy-proofness to the analysis. Beyond impossibility results, social choice theory also studies concrete mechanisms for preference aggregation, such as \emph{voting rules} \cite{taylor2005social}, \emph{scoring rules} \cite{young1975social}, and \emph{judgment aggregation} \cite{list2012theory}, while featuring intersections with other fields such as mechanism design \cite{nisan1999algorithmic} and computer science, the latter resulting in the study of computational social choice \cite{brandt2016handbook}.
Here, we extend the study of social choice to the representative setting, where issues and individuals are too numerous for all preferences to be considered directly. By doing so, \changed{we formalize the problem of democratic representation within a statistical learning framework, providing a learning-theoretic lens for social choice.}

\paragraph{Machine Learning Methods in Social Choice Theory} Machine learning methods have been applied to social choice theory to address limitations of over-simplification. For instance, \emph{generative social choice} studies the problem of handling open-ended outcomes in social choice theory in theoretically sound ways \cite{fish2023generative}, and \emph{dynamic social choice} studies the problem of handling evolving preferences in social choice theory, using theoretical models from reinforcement learning \cite{parkes2013dynamic}. In this paper, we apply the theory of statistical learning to representative social choice, showing that the deepest questions in representative social choice can be formulated as learning problems.

\paragraph{Applications of Social Choice in AI Alignment} Social choice theory has been applied to the study of AI alignment, where the goal is to design AI systems that make decisions that are aligned with human values. Current approaches to AI alignment involves the aggregation of human preferences, and social choice-based algorithms \cite{kopf2024openassistant,klingefjord2024human}, experimental studies \cite{huang2024collective}, conceptual frameworks \cite{prasad2018social,mishra2023ai,conitzerposition,zhi2024beyond}, and axiomatic frameworks \cite{ge2024axioms} have been proposed to address the problem. In this paper, we extend social choice theory to the representative setting, which can then be applied to the AI alignment setting, where human preferences are represented by a collection of individual-issue pairs sampled from human evaluators, and the AI system is trained to make decisions representative of the human population's preferences.

\section{Problem Settings of Representative Social Choice}

In this section, we present the formal definitions of the representative social choice problem, generalizing the standard social choice formalism to include more realistic complexities.

\paragraph{Issues}  We consider a discrete (but possibly infinite) set of $N$-ary issues $\mathcal{I}$, where each issue $i \in \mathcal{I}$ (\emph{e.g.}, in a given state or province, which construction project to launch this year?) comes with $N$ outcomes $[N]=\{1,2,\cdots,N\}$. Each individual's preference profile can therefore be represented as a mapping from $\mathcal{I}$ to $\mathrm{LO}(N)$, where $\mathrm{LO}(N)$ is the set of all \changed{linear orders} over the $N$ outcomes.

We define a \emph{saliency distribution} $\mathcal{D}_{\mathcal{I}}$ with full support over $\mathcal{I}$, which represents the importance of different issues, and decides the probability of each issue being sampled in the representation process. If there are a finite number of equally important issues, then $\mathcal{D}_{\mathcal{I}}$ is the uniform distribution over $\mathcal{I}$. \changed{In practice, $\mathcal{D}_{\mathcal{I}}$ could be estimated from public opinion polls, or in AI alignment, from the empirical distribution of user prompts.}

\paragraph{Population} We consider a possibly infinite population, represented by a distribution $\mathcal{D}_{\mathcal P}\in\Delta[\mathcal P]$ supported on $\mathcal P \subseteq \mathrm{LO}(N)^{\mathcal{I}}$.\footnote{\changed{We focus on linear orders as is standard in classical social choice \cite{arrow2012social}. This framework could be extended to partial orders or cardinal utilities, but we leave this for future work. }$\ldots$ where $\Delta[\mathcal P]$ is the space of distributions over the support set $\mathcal P$. We denote with $A^B$ the space of mappings from $B$ to $A$.} For any preference profile $C\in\mathcal P$, denote with $\mathcal{D}_{\mathcal P}(C)$ the probability (mass or density) that a random individual in the population has preference profile $C$ over the outcomes of all issues.

Often, we only need to consider the marginal distribution of $\mathcal{D}_{\mathcal P}$ --- the mapping $\mathcal{M}:\mathcal I\rightarrow \Delta[\mathrm{LO}(N)]$. %\footnote{$\Delta[\mathrm{LO}(N)]$ is the space of probability distributions over $\mathrm{LO}(N)$}
For any issue $i \in \mathcal{I}$, $\mathcal{M}(i)$ is the distribution of preferences over the $N$ outcomes of issue $i$ in the population. For preference ordering $o \in \mathrm{LO}(N)$, we denote with $\mathcal{M}(i)_o$ the probability that a random individual in the population has preference ordering $o$ over outcomes of issue $i$.

\paragraph{Outcomes} The result of a decision-making process is a preference profile $C:\mathcal{I}\rightarrow \mathrm{LO}(N)$, which represents the aggregated preference of the population generated by some mechanism. 

The mechanism is \emph{representational} if the decision is made based on a finite collection of individual-issue pairs $\mathcal{S} = \{(o_1, i_1), (o_2, i_2), \ldots, (o_{|S|}, i_{|S|})\}$, with $i_k \in \mathcal{I}$ sampled from ${\mathcal D}_{\mathcal I}$, and $o_k \sim \mathcal{M}(i_k)$ is an individual's preference over the outcomes of issue $i_k$, sampled from the population distribution.

However, not all preference profiles are allowed. As a key feature of representative social choice, the mechanism is only allowed to output preference profiles from a limited \emph{candidate space} $\mathcal{C} \subseteq {\mathrm{LO}(N)}^{\mathcal{I}}$ (\emph{e.g.}, mutually compatible combinations of per-state construction projects in the national policy case, or language model policies in the AI alignment case). A mechanism is thus a function $f:{(\mathrm{LO}(N)\times \mathcal I)}^*\rightarrow \mathcal C$, \changed{where $A^*$ denotes the set of all finite sequences of elements from $A$,} that maps the sample collection $\mathcal{S}$ to a preference profile $C\in \mathcal C$.

\changed{\begin{remark}[On Fixed N]
For simplicity, we assume a fixed $N$ for all issues. This can model varying $N_i$ by setting $N = \max(N_i)$ and treating unused outcomes as `dummy' alternatives that are always ranked last.
\end{remark}}

\changed{\begin{remark}[AI Alignment as Representative Social Choice]
	Representative social choice can model the problem of AI alignment, where the LLM is trained to make decisions representative of the human population's preferences.
	Here, the set of \emph{issues} $\mathcal{I}$ is the space of all possible user prompts, and the $N$ \emph{outcomes} are the candidate responses. A \emph{candidate profile} $C \in \mathcal{C}$ is a specific AI policy (\emph{e.g.}, a set of model weights). The \emph{candidate space} $\mathcal{C}$ is the set of all policies achievable via the alignment process (\emph{e.g.}, all models reachable by RLHF training). The \emph{sample} $\mathcal{S}$ is the human preference dataset. Our framework thus analyzes how well a policy trained on $\mathcal{S}$ generalizes to the true human population's preferences over all possible prompts.
\end{remark}}

\paragraph{Binary Setting as Special Case} In the binary ($N=2$) case of representative social choice, each issue is binary (Yes/No), such as in the preference annotations of language model alignment \cite{bai2022training}. This reduction in complexity will allow for the analysis of a well-defined majority vote mechanism.

\section{Summary of Key Results}

Here, we summarize the key results from our analysis of representative social choice.

\changed{\paragraph{Generalization Bounds (Section \ref{sec:binary_generalization_bound_appendix})}
Generalization is essential in representative social choice. It ensures that decisions made by the mechanism based on a finite sample of preferences are representative of the population. We show that for binary issues, the required sample size $|\mathcal{S}|$ to ensure decisions generalize to the whole population scales with the complexity of the candidate space $\mathcal{C}$, a measure of its complexity (Theorem \ref{theorem:bin_gen_bound_appendix}). Theorem \ref{theorem:gen_bound_scoring_appendix} generalizes this bound to non-binary settings.}

In essence, the bounds tell us that the larger the sample size, the better the mechanism's decision reflects the population's true preferences, with the sample size needing to grow in proportion to the complexity of the candidate space, as measured by the VC dimension \cite{vapnik1999overview} or the Rademacher complexity \cite{mohri2008rademacher} --- as an analogy, when election candidates tailor their messaging in a fine-grained manner, the population needs to watch more debates to find broadly aligned candidates, or else candidates could easily \emph{overfit} their message to a few flagship issues.

\changed{\paragraph{Majority Vote and Scoring Mechanisms (Section \ref{sec:majority_vote}, \ref{sec:scoring_rules})}
In the binary setting, this generalization analysis leads to our first result: the simple \emph{majority vote} mechanism is approximately optimal for the whole population, as long as enough samples are drawn (Corollary \ref{cor:majority_vote_appendix}). We extend this to non-binary issues using \emph{scoring mechanisms}, which assign scores to preference profiles. We show these mechanisms also generalize well, with sample requirements depending on the complexity of the issues and outcomes (Corollary \ref{cor:scoring_mechanism_appendix}).}

% In multi-outcome cases, scoring mechanisms provide a more generally applicable approach than majority vote. Like majority vote, scoring mechanisms exhibit good generalization properties as long as the sample size is large enough and in proportion to the complexity of the candidate space.

\changed{\paragraph{Representative Impossibilities (Section \ref{sec:weak_impossibility_theorem}, \ref{sec:privileged_orderings}, \ref{sec:strong_impossibility_theorem})}
Finally, we show that, like its classical counterpart, representative social choice faces fundamental impossibilities. We introduce a weak impossibility theorem (Theorem \ref{theorem:weak_rep_impossibility_appendix}) that directly extends Arrow's theorem. We then strengthen it by introducing \emph{privilege graphs} to analyze issue interdependence, proving a strong impossibility theorem (Theorem \ref{theorem:strong_rep_impossibility_appendix}) that identifies the precise structural conditions---cycles in the privilege graph---that lead to impossibility. It both explains why Arrow's result holds and extends it to the general, constrained-candidate-space setting.}

These theorems show that, in representative social choice, we must make trade-offs between different desirable properties, such as fairness, utility maximization, and convergence. The two theorems differ in their range of applicability, where the latter allows interdependence between issues by lifting the constraint on $\mathcal C$.

% \paragraph{Binary Representative Social Choice (Section \ref{sec:binary_representative})} We start with a simpler case of representative social choice, where an infinitely large population hold preferences over a possibly infinite number of \emph{binary} issues --- issues that can be resolved by a simple Yes/No vote --- and a collection of individual-issue pairs are randomly drawn as samples, from which a collective preference profile\footnote{In this paper, we abuse the term \emph{preference profile} to mean either a collection of preference ordering for different \emph{individuals} in a population, or a collection of preference ordering for different \emph{issues}.} is constructed to represent the entire population. 

% \paragraph{General Representative Social Choice (Section \ref{sec:general_representative})} We then extend our framework to the general case of representative social choice, where the issues can have any finite number of outcomes. We introduce a more general class of mechanisms, the \emph{scoring mechanisms}, 
% and analyze their generalization properties. On the pessimistic side, we present Arrow-like impossibility theorems for representative social choice. 
\vspace{1em}

We consider our contribution in Section \ref{sec:binary_representative}, \ref{sec:scoring_rules}, \ref{sec:weak_impossibility_theorem} to be primarily conceptual and stage-setting, establishing the representative framework with techniques well-known in other fields. Section \ref{sec:privileged_orderings} and \ref{sec:strong_impossibility_theorem}, when establishing the conceptually important strong impossibility theorem, additionally introduce new combinatorial tools of analysis, \emph{privileged orderings} and the \emph{privilege graph}, which we believe to be of independent interest.

\section{Binary Representative Social Choice}\label{sec:binary_representative}

In this section, we consider the case of representative social choice where the issues are binary. % We will show that this setting can be naturally formulated as a statistical learning problem, and we will analyze the generalization properties of the majority vote mechanism under this setting. 
A real-world example is the case of jury trials, where a randomly selected jury makes binary decisions (guilt or innocence) on behalf of the entire population. AI alignment, with its binary preference annotations, is another example.

\subsection{Binary Generalization Bound}\label{sec:binary_generalization_bound_appendix}

When aggregating preferences, the mechanism only sees a finite collection of individual-issue pairs, and therefore the decision is made based on a that finite sample. This raises the question of \emph{generalization} --- how can we ensure that the decision made by the mechanism is representative of the population's preferences, when the mechanism is only optimized for a small number of individual opinions on individual issues?

To answer this question, we can leverage the theory of statistical learning, which studies the generalization properties of learning/optimization algorithms based on finite samples. The reliability of generalization depends on the complexity of the candidate space $\mathcal{C}$ --- the more flexibile the candidate space, the more likely a profile can be picked by the mechanism that specifically fits the sample (\emph{overfitting}) rather than the broader population. To characterize complexity, we use \emph{Vapnik-Chervonenkis (VC) dimension} (Definition \ref{def:vc_dimension_appendix}) to measures the capacity of a hypothesis space to fit arbitrary finite samples \cite{vapnik1999overview}.

We can now introduce the sample complexity theorem. For any preference profile $C$ as the aggregation outcome, the theorem gives an upper bound on the difference between the \emph{sample utility} $\frac 1{|\mathcal S|} \sum_k \mathbf{1}_{C(i_k)={o_k}}$ (the goodness of the aggregated profile, evaluated on the selected samples) and the \emph{population utility} $\mathrm{E}_{i\sim {\mathcal D}_{\mathcal I}}[\mathcal M(i)_{C(i)}]$ (the unknown utility of the aggregated profile for the entire population), as a function of the sample size $|\mathcal S|$ and the VC dimension of the candidate space $\mathcal C$. Such a difference is called the \emph{generalization error}.

\begin{theorem}[Binary Generalization Bound]\label{theorem:bin_gen_bound_appendix}
	Let $\mathcal{C}$ be a candidate space with VC dimension $\mathrm{VC}(\mathcal C)$, and let $\epsilon > 0$ be a desired generalization error. Then, for any $\delta > 0$, with probability at least $1-\delta$, the sample utility and population utility of \emph{any} preference profile $C \in \mathcal C$ are $\epsilon$-close, \emph{i.e.},
	\begin{equation}
		\mathrm{Pr}\left[
		\left|\frac 1{|\mathcal S|} \sum_{k=1}^{|\mathcal S|} \mathbf{1}_{C(i_k)={o_k}} - \mathrm{E}_{i\sim {\mathcal D}_{\mathcal I}}[\mathcal M(i)_{C(i)}]\right| \leq \epsilon,
		\quad \forall C \in \mathcal C
		\right]\geq 1-\delta
	\end{equation}
	as long as we have the following, for some constant $c > 0$:
	\begin{equation}
		|\mathcal S| \geq \frac{c}{\epsilon^2} \mathrm{VC}(\mathcal C)\left(\log\mathrm{VC}(\mathcal C) + \log\frac 1\epsilon + \log\frac{1}{\delta}\right)
	\end{equation}
\end{theorem}

\begin{remark}[On the Role of Generalization Bounds]
\changed{While Theorem \ref{theorem:bin_gen_bound_appendix} is an application of established Vapnik-Chervonenkis (VC) theory, its importance here is conceptual. It formally connects the `representativeness' of a social choice mechanism to the \emph{complexity of the candidate space} $\mathcal{C}$. This quantifies a trade-off, where a more expressive set of possible candidates (\emph{e.g.}, more nuanced political platforms or AI policies) requires significantly more preference samples to ensure the chosen candidate truly represents the population.}
\end{remark}

Intuitively, the sample complexity theorem states that the sample utility of any preference profile in the candidate space $\mathcal C$ approximates the population utility, as long as the sample size is sufficiently large and in proportion to the complexity of the candidate space.
% The sample size required for this guarantee depends on the VC dimension of the candidate space, the desired generalization error $\epsilon$, and the desired confidence level $\delta$.
Note that here we cannot directly apply the tail inequalities \cite{hellman1970probability} to bound the generalization error, because when the profile $C$ is picked to maximize sample utility, the sample utility ceases to be an unbiased estimator of the population utility. 

% Theorem \ref{theorem:bin_gen_bound_appendix} will be central in Section \ref{sec:majority_vote}, where we analyze the generalization properties of the majority vote mechanism under the binary representative setting.

\subsection{Case Study: Majority Vote in the Binary Case}\label{sec:majority_vote}

In this section, we consider the \emph{majority vote} mechanism under the binary representative setting. It generalizes the well-known majority vote mechanism where the population directly votes a single issue, to the representative setting where the population is represented by a collection of individual-issue pairs.

\begin{definition}[Majority Vote Mechanism]
	The majority vote mechanism is a representational mechanism $f_{\mathrm{maj}}$ that outputs the preference profile $C$ that maximizes the sample utility, \emph{i.e.},
	\begin{equation}
		f_{\mathrm{maj}}(\mathcal S) = \arg\max_{C \in \mathcal C} \frac 1{|\mathcal S|} \sum_{k=1}^{|\mathcal S|} \mathbf{1}_{C(i_k)={o_k}}.
	\end{equation}
\end{definition}

The majority vote mechanism can be viewed as a voting process where each individual-issue pair in the sample collection $\mathcal{S}$ casts a vote for the outcome that the individual prefers, based on the individual's preference over the two outcomes of the issue. The candidate profile that receives the most votes is then selected as the aggregated preference profile. When there is only one issue and candidate space $\mathcal{C} = \mathrm{LO}(2)$, the majority vote mechanism reduces to the standard majority vote mechanism in social choice theory.

From Theorem \ref{theorem:bin_gen_bound_appendix}, we know that the majority vote mechanism has good generalization properties when the VC dimension of the candidate space $\mathcal{C}$ is small, resulting in the following corollary.

\begin{corollary}[Majority Vote Approximately Maximizes Population Utility]\label{cor:majority_vote_appendix}
	For any error requirement $\epsilon>0$ and confidence requirement $\delta>0$, the majority vote mechanism $f_{\mathrm{maj}}$ satisfies
	\begin{equation}
		\mathrm{Pr}\left[
		U_{\mathcal{D}_{\mathcal I},\mathcal{M}}(f_{\mathrm{maj}}(\mathcal S))
		\geq
		-2\epsilon + 
		\max_{C \in \mathcal C} U_{\mathcal{D}_{\mathcal I},\mathcal{M}}(C)
		\right]\geq 1-\delta
	\end{equation}
	where the population utility
	\begin{equation}
		U_{\mathcal{D}_{\mathcal I},\mathcal{M}}(C) \coloneq \mathrm{E}_{i\sim {\mathcal D}_{\mathcal I}}[\mathcal M(i)_{C(i)}],
	\end{equation}
	as long as the sample size $|\mathcal S|$ satisfies
	\begin{equation}
		|\mathcal S| \geq \frac{c}{\epsilon^2} \mathrm{VC}(\mathcal C)\left(\log\mathrm{VC}(\mathcal C) + \log\frac 1\epsilon + \log\frac{1}{\delta}\right)
	\end{equation}
\end{corollary}

% Corollary \ref{cor:majority_vote_appendix} shows that the majority vote mechanism approximately maximizes the population utility, as long as the sample size is sufficiently large and the VC dimension of the candidate space is small. 

Furthermore, we can verify that the majority vote mechanism satisfies the following formal axioms, analogous to the classical \emph{Pareto efficiency} and \emph{non-dictatorship} axioms in social choice theory \cite{fishburn2015theory}.

\begin{axiombox}{Probabilistic Pareto Efficiency (PPE), Binary Case}
	Consider profiles $C,C'\in \mathcal C$ such that for the only issue $i\in\mathcal I$ that they disagree on, we have $1\succ_C 0, 0\succ_{C'} 1$ and the population is unanimous on $1\succ 0$.\footnote{In other words, $\mathcal{M}(i)_{1\succ 0}=1$} For any such $C,C'$ and a sufficiently large sample size $|\mathcal S|$, with probability at least $1-e^{\alpha|\mathcal S|}$,\footnote{\ldots where $\alpha>0$ is an arbitrary constant dependent only on $C,C'$ \changed{(we use $1-e^{\Omega(|\mathcal S|)}$ to denote a probability that converges to 1 exponentially fast in $|\mathcal S|$, where $\Omega(|\mathcal S|)$ is any quantity that asymptotically grows no slower than $|\mathcal S|$)}, and $|\mathcal S|$ is the number of samples. From now on, we will abbreviate this expression to $1-e^{\Omega(|\mathcal S|)}$.} $f_{\mathrm{maj}}(\mathcal S) \neq C'$. Likewise when $0\succ_{C} 1, 1\succ_{C'} 0$ and the population is unanimous on $0\succ 1$.
\end{axiombox}

\begin{remark}
	Introduction of the candidate space $\mathcal C$ leads to interdependence between issues, which is not present in the classical social choice setting. As a result, we could not simply require that the mechanism output $1\succ 0$ when the population is unanimous on $1\succ 0$ as in the classical setting --- what if the population is unanimous on $1\succ 0$ for one issue, but unanimous on $0\succ 1$ for another issue, and the two issues are strongly correlated (\emph{e.g.}, every candidate profile agrees on the two issues)? Cases like these, while less extreme, are widespread in the real world. Our fomulation of PPE avoids this problem by comparing two candidate profiles $C$ and $C'$ that keep the same preference on all issues except one, a way to ensure ``all else being equal''.

	$1-e^{\Omega(|\mathcal S|)}$ is the convergence rate from Hoeffding's inequality for independent samples. Intuitively, it's the confidence with which you can tell a majority from a minority in the population based on a sample $\mathcal S$.
\end{remark}

\begin{axiombox}{Probabilistic Non-Dictatorship (PND), Binary Case}
	For any issue $i\in\mathcal{I}$, for any subpopulation ${\mathcal{D}'}_\mathcal{P}$ that occupies a probability mass $|{\mathcal{D}'}_\mathcal{P}|<0.5$ in the whole population $\mathcal{D}_\mathcal{P}$, at least one of below is true w.r.t. issue $i$:
  \vspace{-0.5em}
	\begin{itemize}[leftmargin=2em]
    \setlength\itemsep{0.1em}
		\item When ${\mathcal{D}'}_\mathcal{P}$ is unanimous on $0\succ 1$, there exists a \changed{population distribution} $\mathcal{D}_\mathcal{P}$ for which $f_{\mathrm{maj}}(\mathcal S)(i) = (1\succ 0)$ with probability $1-e^{\Omega(|\mathcal S|)}$.
		\item When ${\mathcal{D}'}_\mathcal{P}$ is unanimous on $1\succ 0$, there exists a \changed{population distribution} $\mathcal{D}_\mathcal{P}$ for which $f_{\mathrm{maj}}(\mathcal S)(i) = (0\succ 1)$ with probability $1-e^{\Omega(|\mathcal S|)}$.
	\end{itemize}
  \vspace{-0.5em}
\end{axiombox}

\begin{remark}
	Here, the requirement that \emph{at least one} (as opposed to \emph{both}) condition is fulfilled is, again, a consequence of the candidate space $\mathcal C$. In the extreme case, if all candidate profiles agree that $0\succ 1$ for issue $i$, then at most one of the two conditions can be met.
\end{remark}

The PND axiom above is stronger than the classical version, in the sense that not only individuals, but also coalitions, cannot dictate the aggregation result. In fact, our problem setting treats individuals as interchangeable, thereby implicitly forcing \emph{anonymity} of the mechanism --- a stronger property than PND. As a result, representational mechanisms always satisfy PND, a fact that is formalized in Lemma \ref{lem:pnd_all_mechanisms_appendix}.

Note that the classical axiom of independence of irrelevant alternatives (IIA) \cite{fishburn2015theory} is not applicable to the binary case, since irrelevant alternatives only exist when $N\geq 3$.

\section{General Representative Social Choice}\label{sec:general_representative}

Having formulated and analyzed the binary case of representative social choice, we now extend our analysis to the general case where there can be an arbitrary number of outcomes for each issue. 

\subsection{Scoring Rules and Generalization Errors}\label{sec:scoring_rules}

Generalizing the majority vote mechanism to the general case meets challenges, as the majority vote is over candidate profiles as opposed to outcomes, and when $N>2$, the way each individual-issue pair votes is no longer well-defined. Instead, we can consider a more general class of mechanisms called \emph{scoring mechanisms} \cite{lepelley2000scoring}, which for every individual-issue pair, assign a score to each candidate profile, and then output the candidate profile with the highest average score.

As an example, the scoring rule could be an arbitrary distance measure that measures the degree of alignment between the sampled individual's preference and the candidate profile's preference. The mechanism then outputs the candidate profile that maximizes the average sample alignment score.

\begin{definition}
	A scoring mechanism is defined by a \emph{scoring rule} $s:\mathrm{LO}(N)\times\mathrm{LO}(N)\rightarrow\mathbb R$, which assigns a score to each pair of preference orderings over the $N$ outcomes. The mechanism then outputs the candidate profile that maximizes the average score over all individual-issue pairs in the sample collection $\mathcal{S}$, \emph{i.e.},
	\begin{equation}
		f_s(\mathcal S) = \arg\max_{C \in \mathcal C} \frac 1{|\mathcal S|} \sum_{k=1}^{|\mathcal S|} s(o_k, C(i_k)).
	\end{equation}
\end{definition}

For scoring mechanisms, we can similarly define the sample score $\frac 1{|\mathcal S|} \sum_{k=1}^{|\mathcal S|} s(o_k, C(i_k))$ and the population score $U^s_{{\mathcal D}_{\mathcal I},\mathcal M}\coloneq\mathrm{E}_{i\sim {\mathcal D}_{\mathcal I},o\sim \mathcal M(i)}[s(o, C(i))]$, and have the following guarantees on generalization.%, analogous to Corollary \ref{cor:majority_vote_appendix}.

\begin{theorem}[Generalization Bound for Scoring Mechanisms]\label{theorem:gen_bound_scoring_appendix}
	For any scoring rule $s$ with value bounded by constants, the scoring mechanism $f_s$ satisfies
	\begin{equation}
		\mathrm{Pr}\left[
		U^s_{{\mathcal D}_{\mathcal I},\mathcal M}(f_s(\mathcal S))
		\geq
		-2\epsilon + 
		\max_{C \in \mathcal C} U^s_{{\mathcal D}_{\mathcal I},\mathcal M}(C)
		\right]\geq 1-\delta
	\end{equation}
	as long as the sample size $|\mathcal S|$ satisfies
	\begin{equation}
      |\mathcal S| \geq \frac{c}{\epsilon^2} \log\frac{1}{\delta}
,\ 
      {\mathrm{\hat{R}}}_{|\mathcal S|}(\bar{\mathcal C}) \leq \frac{\epsilon}{c},
  \text{ where }
      \bar{\mathcal C} \coloneq \{(o, i) \mapsto s(o, C(i)) \mid C \in \mathcal C\}
  \end{equation}
    and ${\mathrm{\hat{R}}}_{|\mathcal S|}(\bar{\mathcal C})$ is the \emph{empirical Rademacher complexity} of $\bar{\mathcal C}$ with respect to the sample collection $\mathcal S$ \cite{mohri2008rademacher} --- a generalization of the VC dimension to real-valued (as opposed to binary) functions.
\end{theorem}
% \begin{proof}
    The result follows directly from the Rademacher generalization bounds in statistical learning theory.
% \end{proof}

\begin{corollary}[Scoring Mechanisms Approximately Maximize Population Score]\label{cor:scoring_mechanism_appendix}
	When $|{\mathcal I}|$ is finite, for any scoring rule $s$ with value bounded by constants, the scoring mechanism $f_s$ satisfies
	\begin{equation}
		\mathrm{Pr}\left[
		U^s_{{\mathcal D}_{\mathcal I},\mathcal M}(f_s(\mathcal S))
		\geq
		-2\epsilon + 
		\max_{C \in \mathcal C} U^s_{{\mathcal D}_{\mathcal I},\mathcal M}(C)
		\right]\geq 1-\delta
	\end{equation}
    as long as we have the following, for some constant $c > 0$:
	\begin{equation}
		|\mathcal S| \geq \frac{c}{\epsilon^2} \left(|{\mathcal I}|N\log N + \log\frac{1}{\delta}\right)
	\end{equation}
\end{corollary}
	
Rademacher complexity generalizes the VC dimension to real-valued functions, and measures the capacity of a function class to fit arbitrary finite samples. Corollary \ref{cor:scoring_mechanism_appendix} states that the sample score of any candidate profile in the candidate space $\mathcal C$ approximates the population score, as long as the sample size exceeds the ability of the candidate space to fit arbitrary finite samples --- \emph{i.e.}, when the mechanism is forced to \emph{generalize}.

Corollary \ref{cor:scoring_mechanism_appendix} is agnostic towards the correlation structure among issues, and as a result, the sample complexity grows approximately linearly with the number of issues $|\mathcal I|$ and the number of outcomes $N$. When given a issue correlation structure, the sample complexity can potentially be reduced using Theorem \ref{theorem:gen_bound_scoring_appendix}.

\subsection{Weak Representative Impossibility}\label{sec:weak_impossibility_theorem}

In this section, we present axioms that ideal social choice mechanisms should satisfy, and present Arrow-like impossibility theorems, showing that no mechanism can satisfy all these axioms simultaneously. A weaker version of the impossibility theorem --- which we present in this section --- is a simple generalization of Arrow's impossibility theorem, while a stronger version shall be derived in Section \ref{sec:strong_impossibility_theorem}. We introduce the necessary notations and then restate the axioms from the binary case, adapting them to our general setting.

\begin{definition}[Operation of a Permutation on an Ordering]
	For a set $S$, a linear order $o\in\mathrm{LO}(S)$, and a permutation $\sigma\in\mathfrak{S}_S$, we define the operation $o\odot\sigma$ as the ordering obtained by applying the permutation $\sigma$ to the elements of $o$. Specifically, for any $s_1,s_2\in S$, we have $s_1\succ_{o\odot\sigma} s_2$ if and only if $\sigma^{-1}(s_1)\succ_o \sigma^{-1}(s_2)$.

	For a subset $T\subseteq S$ and $\sigma\in \mathfrak{S}_T$, we similarly define $o\odot\sigma\coloneqq o\odot\sigma|_S$, where $\sigma|_S$ is the extension of $\sigma$ to the whole set $S$, mapping elements outside $T$ to themselves. For a full profile $C\in\mathrm{LO}(N)^{\mathcal I}$, we define $C\odot_i\sigma$ as the profile obtained by applying the permutation $\sigma$ to $C(i)$, while keeping $C$ unchanged for all other issues.
\end{definition}
\begin{remark}
	We denote with $\sigma_{(c_1, c_2)}$ the $2$-element permutation (transposition) that swaps $c_1$ and $c_2$. Operation of a permutation on an ordering aims to capture ``local changes'' to a preference profile, where the permutation only affects a subset of the outcomes. They will be used in the definition of the axioms below. 
\end{remark}

\begin{axiombox}{Probabilistic Pareto Efficiency (PPE)}
	Consider profiles $C,C'\in \mathcal C$ such that for the only issue $i\in\mathcal I$ that they disagree on, there exists $c,c'\in [N]$ such that $c\succ_C c', C'=C\odot_i\sigma_{(c,c')}$\footnote{The second condition means that the order between $c,c'$ is the only disagreement they have on issue $i$.} and the population is unanimous on $c\succ c'$.\footnote{In other words, $\mathcal{M}(i)_{o} = 0$ for all $o\in \mathrm{LO}(N)$ that prefers $c'$ over $c$.} For any such $C,C'$, with probability $1-e^{\Omega(|\mathcal S|)}$, $f(\mathcal S) \neq C'$.
\end{axiombox}

% \begin{axiombox}{Probabilistic Non-Dictatorship (PND)} 
% 	Consider profiles $C,C'\in \mathcal C$ such that for the only issue $i\in\mathcal I$ that they disagree on, there exists $c,c'\in [N]$ such that $c\succ_C c', C'=C\odot_i\sigma_{(c,c')}$. For any subpopulation ${\mathcal{D}'}_\mathcal{P}$ that's unanimous on $c'\succ c$ and that occupies a probability mass strictly less than $0.5$ in the whole population $\mathcal{D}_\mathcal{P}$, there exists a preference specification $\mathcal{D}_\mathcal{P}$ of the whole population for which $f(\mathcal S) \neq C'$ with probability $1-e^{\Omega(|\mathcal S|)}$.
% \end{axiombox}

\begin{axiombox}{Probabilistic Non-Dictatorship (PND)}
	For any issue $i\in\mathcal{I}$ and $c,c'\in[N]$, for any subpopulation ${\mathcal{D}'}_\mathcal{P}$ that occupies a probability mass $|{\mathcal{D}'}_\mathcal{P}|<0.5$ in the whole population $\mathcal{D}_\mathcal{P}$, at least one of below is true w.r.t. issue $i$:
  \vspace{-0.5em}
	\begin{itemize}[leftmargin=2em]
    \setlength\itemsep{0.1em}
		\item When ${\mathcal{D}'}_\mathcal{P}$ is unanimous on $c\succ c'$, there exists a \changed{population distribution} $\mathcal{D}_\mathcal{P}$ for which $c'\succ_{f(\mathcal S)} c$ with probability $1-e^{\Omega(|\mathcal S|)}$.
		\item When ${\mathcal{D}'}_\mathcal{P}$ is unanimous on $c'\succ c$, there exists a \changed{population distribution} $\mathcal{D}_\mathcal{P}$ for which $c\succ_{f(\mathcal S)} c'$ with probability $1-e^{\Omega(|\mathcal S|)}$.
	\end{itemize}
  \vspace{-0.5em}
\end{axiombox}

In fact, \emph{all} representational mechanisms satisfy PND, since our setting treats individuals as interchangeable and homogeneous, therefore implicitly forcing \emph{anonymity} of the mechanism, which inturn implies PND. % We formalize this fact in the following lemma, which will turn out useful in the proof of the later, strong version of the impossibility theorem.

\begin{lemma}[Probabilistic Non-Dictatorship for All Representational Mechanisms]\label{lem:pnd_all_mechanisms_appendix}
	For any $(\mathcal I, \mathcal D_{\mathcal I}, \mathcal C)$ and any representational mechanism $f$, PND is satisfied.
\end{lemma}

As a result, we shall remove PND from the explicit statements of the impossibility theorems, but it should be understood that PND is still implicitly satisfied. Finally, we define a weak version of the independence of irrelevant alternatives (IIA) axiom, as well as a new axiom specific to the representative setting. 

\begin{axiombox}{Weak Probabilistic Independence of Irrelevant Alternatives (W-PIIA)}
	When $\mathcal C = \mathrm{LO}(N)^{\mathcal I}$, for two populations $\mathcal{D}_{\mathcal P},\mathcal{D}'_{\mathcal P}$ that differ only in the preference over a single issue $i\in\mathcal I$ satisfying $\mathcal{M}(i)\mid_{\mathrm{LO}(\{c,c'\})} = \mathcal{M}'(i)\mid_{\mathrm{LO}(\{c,c'\})}$\changed{\footnote{\changed{Here, $\mathcal{M}(i)\mid_{\mathrm{LO}(\{c,c'\})}$ denotes the marginal distribution of preferences when restricted to the set $\{c,c'\}$.}}} (where $c,c'$ are any two elements of $[N]$),\footnote{Meaning that the population's distribution over the preference between $c,c'$ are the same in the two populations.} with probability $1-e^{\Omega(|\mathcal S|)}$, we have $f(\mathcal S)\mid_{\mathrm{LO}(\{c,c'\})}=f(\mathcal S')\mid_{\mathrm{LO}(\{c,c'\})}$.
\end{axiombox}

\begin{remark}
	\changed{The assumption $\mathcal C = \mathrm{LO}(N)^{\mathcal I}$ implies a \emph{structural independence} in the candidate space: any preference over one issue can be combined with any preference over any other. As a result, we avoid ``cross-issue'' or ``cross-outcome'' \emph{structural} dependencies; for example, we don't worry that a population's insistence on $c_1\succ c_2$ might force the side effect of $c_3 \succ c_4$, simply because all available profiles in $\mathcal C$ happen to bundle $\{c_1,c_3\}$ together and $\{c_2,c_4\}$ together. This structural independence of $\mathcal{C}$ should be distinguished from \emph{statistical correlations} in the population's preferences $\mathcal{D}_{\mathcal P}$ (which are not restricted by this assumption). This simplification makes $\mathcal C = \mathrm{LO}(N)^{\mathcal I}$ the easiest candidate space to analyze. We will lift this assumption in Section \ref{sec:privileged_orderings} and \ref{sec:strong_impossibility_theorem}.}
\end{remark}

\begin{axiombox}{Weak Probabilistic Convergence (W-PC)}
	When $\mathcal C = \mathrm{LO}(N)^{\mathcal I}$, for a population $\mathcal{D}_{\mathcal P}$ that's non-uniform on outcomes $c,c'$ of issue $i\in\mathcal I$,\footnote{\emph{i.e.}, $\mathcal{M}(i)\mid_{\mathrm{LO}(\{c,c'\})}(c\succ c')\neq 0.5$. We are only concerned with the marginal distribution.} there exists an ordering $o\in\mathrm{LO}(\{c,c'\})$ such that with probability $1-e^{\Omega(|\mathcal S|)}$, $f(\mathcal S)(i)\mid_{\mathrm{LO}(\{c,c'\})}=o$.
\end{axiombox}

\begin{remark}
	W-PC requires that the mechanism not be torn between two outcomes when the population is not torn between them. This is to rule out ``indecisive'' mechanisms that are unable to make a decision with high probability (or requires too many samples to make that decision) when the population has clear preference. This need arises only in the representative setting where the mechanism is no longer deterministic.
	% It's worth noting that $1-e^{\Omega(|\mathcal S|)}$ is the convergence rate guaranteed by Hoeffding's inequality for independent samples, which, intuitively speaking, asks that the mechanism not be qualitatively slower in convergence than the majority vote mechanism. And again, $\mathcal C = \mathrm{LO}(N)^{\mathcal I}$ removes the complexity of cross-outcome dependencies, which simplifies the statement of the axiom.
\end{remark}

\begin{theorem}[Weak Representative Impossibility]\label{theorem:weak_rep_impossibility_appendix}
	When $N\geq 3$, for any $(\mathcal I, \mathcal D_{\mathcal I}, \mathcal C = \mathrm{LO}(N)^{\mathcal I})$, no representational mechanism simultaneously satisfies PPE, W-PIIA, and W-PC for all $\mathcal D_{\mathcal P}$.
\end{theorem}

The weak representative impossibility theorem is ``weak'' in the sense that it assumes $\mathcal{C} = \mathrm{LO}(N)^{\mathcal I}$, \emph{i.e.}, no interdependence between issues or outcomes exist. This shall change in the next section.

\subsection{Privileged Orderings and Privilege Graph}\label{sec:privileged_orderings}

Before we present the strong representative impossibility theorem, we need tools to represent and analyze the structure of the candidate space $\mathcal C$. To this end, we introduce the concept of \emph{privileged orderings}, which are partial orderings that are preferred over all other alternative orderings in the candidate space.

\begin{definition}[Privileged Ordering]
	For an issue $i\in\mathcal I$ and a subset of outcomes $T=\{c_1,c_2,\cdots,c_k\}\subseteq[N]\ (k\geq 2)$, we call $o\in \mathrm{LO}(T): c_1\succ c_2\succ\cdots\succ c_k$ a \emph{privileged ordering} if for any extension $C\in \mathrm{LO}(N)^{\mathcal I}$ of $o$\footnote{\emph{i.e.}, $C$ must agree with $o$ on issue $i$ that $c_1\succ\cdots\succ c_k$, but can differ on the remaining outcomes of $i$ and on other issues.} and permutation $\sigma\in\mathfrak{S}_T$ of $T$, we have $C\odot_i \sigma\in \mathcal{C} \implies C\in \mathcal{C}$.
\end{definition}
\begin{remark}
	Intuitively, a privileged ordering $o$ is a partial ordering that is preferred in the candidate space $\mathcal C$ over all other alternative partial orderings. Any full preference profile in $\mathcal C$ that disagrees with $o$ must have a counterpart in $\mathcal C$ that agrees with $o$ while keeping the rest of the profile unchanged. These privileged orderings are in fact surprisingly common in practice, as will be showcased in Example \ref{example:privilege_graph} and \ref{example:privilege_graph_2}.
\end{remark}

\begin{definition}[Privilege Graph]\label{def:privilege_graph}
	For an issue $i\in\mathcal I$, we define the \emph{privilege graph} $G_i$ as a directed graph with vertices as the outcomes in $[N]$, and an edge from $u$ to $v$ iff there exists a privileged ordering $u\succ v$.
	We call $i$ \emph{cyclically privileged} if its privilege graph contains a simple directed cycle of length at least $3$.\footnote{Note that for our purposes, we don't consider graphs containing only $2$-cycles as cyclic.}
\end{definition}

\begin{example}[Privileged Orderings in the Real World]\label{example:privilege_graph}
  Consider three persons $\mathcal{I}=\{\text{A}, \text{B}, \text{C}\}$ on trial before the jury. A is charged of burglary (\$5k, first trial), B also of burglary (\$50k, second trial), and C of fraud. The $N=3$ possible outcomes for each defendant include aqcuittal ($\mathbf{a}$), community service ($\mathbf{c}$), and imprisonment ($\mathbf{i}$). The jury ranks the outcomes for each defendant in order of recommendation. In this hypothetical case, the following factors may lead to privileged orderings:
  \vspace{-0.5em}
	\begin{itemize}[leftmargin=2em]
    \setlength\itemsep{0.1em}
		\item \textbf{Monotonicity and fairness constraints.} The jury may consider it unfair to punish A harder than B given the difference in the amount of burglary. Thus, locally changing the recommended outcome of A from $\mathbf{c}$ to $\mathbf{i}$ may violate this monotonicity (namely when the outcome of B is $\mathbf{c}$), but changing from $\mathbf{i}$ to $\mathbf{c}$ will not. This implies $\mathbf{c}\succ \mathbf{i}$ and (analogously) $\mathbf{a}\succ \mathbf{c}$ being privileged for A, and $\mathbf{i}\succ \mathbf{c},\mathbf{c}\succ \mathbf{a}$ for B.
		\item \textbf{Local independence.} While the crimes of A and B are similar and therefore correlated in the jury's judgment, C is unrelated to the other two. As a result, $\mathcal C$ may be the Cartesian product of the candidate space ${\mathcal C}_{(\text{A},\text{B})}$ for $(\text{A},\text{B})$ and ${\mathcal C}_C$ for C. When ${\mathcal C}_C=\mathrm{LO}(\{\mathbf{a},\mathbf{c},\mathbf{i}\})$ and $\mathcal C={\mathcal C}_{(\text{A},\text{B})}\times \mathrm{LO}(\{\mathbf{a},\mathbf{c},\mathbf{i}\})$, it can be verified that C's privilege graph $G_C$ is a complete digraph, and every ordering is privileged for C.
		\item \textbf{Default outcomes.} Assume that B is sentenced to $\mathbf{c}$ in the first trial, and the jury is inclined to keep the sentence since no new substantial evidence is presented. This may lead to the privileged orderings $\mathbf{c}\succ \mathbf{i}$ and $\mathbf{c}\succ \mathbf{a}$ for B, since $\mathbf{c}$, being the default outcome, is a plausible substitute for any other outcome.
	\end{itemize}
  \vspace{-0.5em}
	These considerations are general, going far beyond the example's legal setting. Locally independent decisions, and decisions with a default outcome, are common in the real world; and many decisions (\emph{e.g.}, those involving numerical balancing of costs and benefits), according to common sense, should be monotonic as well. There are other reasons for privileged orderings too, and the three factors listed above are only for illustration.
\end{example}

\begin{example}[Privileged Orderings in AI and AI Alignment]\label{example:privilege_graph_2}
	\changed{Outside of the human case, privileged orderings (Definition \ref{def:privilege_graph}) are a natural model for the \emph{inductive biases} of AI systems \cite{baxter2000model}, including deep neural networks. For instance, an LLM pretrained on web text might have a structural bias (a `privilege') for certain types of responses (\emph{e.g.}, coherent ones) over others \cite{perez2022discovering,santurkar2023whose}, and alignment training may be unable to make the model output a response that violates a deep-seated bias; this creates a non-trivial candidate space $\mathcal C \subsetneq \mathrm{LO}(N)^{\mathcal I}$ where impossibility results can arise.} More examples include:
  \vspace{-0.5em}
	\begin{itemize}[leftmargin=2em]
    \setlength\itemsep{0.1em}
		\item \textbf{Language models}. Language models are known to display a wide range of human behavioral tendencies \cite{perez2022discovering,santurkar2023whose,lampinen2024language}, including ones discussed in Example \ref{example:privilege_graph}. Tendencies including consistent treatment of similar decisions, default outcomes, monotonicity, and many more, can all lead to privileged orderings. In the alignment process, these pretrained tendencies serve as inductive biases that set limits on the possible outcomes of alignment training.
		\item \textbf{Reinforcement learning agents}. In reinforcement learning, the reward function is a key component that guides the learning process. It has been demonstrated to induce agents to learn deeply ingrained hiearchies or equivalences between outcomes \cite{di2022goal,wang2024investigating}. These prior tendencies set limits on the possible outcomes of later learning, leading to privileged orderings.
		\item \textbf{Vision models}. Vision models have been shown to exhibit spatial locality \cite{li2022locality}, translation invariance \cite{kauderer2017quantifying}, simplicity bias \cite{shah2020pitfalls}, and other inductive biases. These biases, which are often necessary for the models to learn effectively, operate by setting limits on the possible outcomes of learning, leading to privileged orderings.
	\end{itemize}
  \vspace{-0.5em}
	While these tendencies in AI systems are often probabilistic and less clear-cut than in the human case, privileged orderings can still be a useful abstraction for understanding and approximating their behavior.
\end{example}

Next, we characterize important properties of privileged orderings and privilege graphs.

\begin{lemma}[Transitivity of Privileged Graph]\label{lem:transitivity_privilege_appendix}
	For an issue $i\in\mathcal I$ and $u, v, w\in [N]$, if $u\succ v$ and $v\succ w$ are both privileged, then $u\succ w$ is privileged.
\end{lemma}

\begin{lemma}[Closure of Privileged Orderings Under Concatenation]\label{lem:closure_privilege_appendix}
	Let $c_1,c_2,\cdots,c_{k+m-1}\in [N]\ (k,m\geq 2)$ be distinct outcomes in an issue $i$. If $o: c_1\succ c_2\succ\cdots\succ c_k$ and $o': c_k\succ c_{k+1}\succ c_{k+2}\succ\cdots\succ c_{k+m-1}$ are both privileged orderings, then $o\oplus o': c_1\succ c_2\succ\cdots\succ c_{k+m-1}$ is also a privileged ordering.
\end{lemma}

\begin{corollary}
	\label{cor:graph_privilege_appendix}
	For issue $i\in\mathcal I$ and $o: c_1 \succ c_2 \succ \cdots \succ c_k\ (k\geq 2)$ containing distinct outcomes in $i$, $o$ is a privileged ordering if a (possibly non-simple) path exists in $G_i$ that passes through $c_1, c_2, \cdots, c_k$ in that order.
\end{corollary}

The privilege graph $G_i$ captures all binary privileged orderings in issue $i$, and will be the primary way in which we represent the structure of the candidate space $\mathcal C$. Properties stronger than Lemma \ref{lem:transitivity_privilege_appendix} and \ref{lem:closure_privilege_appendix} often do not hold, \emph{e.g.}, a subsequence of a privileged ordering is not necessarily privileged. Also, the condition in Corollary \ref{cor:graph_privilege_appendix} is a sufficient but not necessary condition for a privileged ordering. 

% With these preparations, we are now able to state and prove the strong representative impossibility.

\subsection{Strong Representative Impossibility}\label{sec:strong_impossibility_theorem}

Before stating the strong impossibility, we first present the strong version of the PIIA and PC axioms.
% They imply their weak counterparts while being applicable to arbitrary candidate spaces $\mathcal C$ and thus more general.

\begin{axiombox}{Strong Probabilistic Independence of Irrelevant Alternatives (S-PIIA)}
	For arbitrary $\mathcal C$ and $c,c'\in[N]$ such that $c\succ c'$ and $c'\succ c$ are both privileged orderings in some issue $i\in\mathcal{I}$, for two populations $\mathcal{D}_{\mathcal P},\mathcal{D}'_{\mathcal P}$ satisfying $\mathcal{M}(i)\mid_{\mathrm{LO}(\{c,c'\})} = \mathcal{M}'(i)\mid_{\mathrm{LO}(\{c,c'\})}$,\footnote{Meaning that the population's distribution over the preference between $c,c'$ are the same in the two populations.} with probability $1-e^{\Omega(|\mathcal S|)}$, we have $f(\mathcal S)\mid_{\mathrm{LO}(\{c,c'\})}=f(\mathcal S')\mid_{\mathrm{LO}(\{c,c'\})}$.
\end{axiombox}

\begin{axiombox}{Strong Probabilistic Convergence (S-PC)}
	For arbitrary $\mathcal C$ and $c,c'\in[N]$ such that $c\succ c'$ and $c'\succ c$ are both privileged orderings in some issue $i\in\mathcal{I}$, for a population $\mathcal{D}_{\mathcal P}$ that's non-uniform on $\{c,c'\}$,\footnote{\emph{i.e.}, $\mathcal{M}(i)\mid_{\mathrm{LO}(\{c,c'\})}(c\succ c')\neq 0.5$. We are only concerned with the marginal distribution.} there exists an ordering $o\in\mathrm{LO}(\{c,c'\})$ such that with probability $1-e^{\Omega(|\mathcal S|)}$, $f(\mathcal S)(i)\mid_{\mathrm{LO}(\{c,c'\})}=o$.
\end{axiombox}

\begin{remark}
	S-PIIA and S-PC are generalizations of W-PIIA and W-PC, respectively, to arbitrary candidate spaces $\mathcal C$. This is achieved by limiting comparisons to pairs that are privileged orderings in both directions, ensuring that they are locally independent from other outcomes and issues.
\end{remark}

This finally leads us to the strong representative impossibility theorem.

\begin{theorem}[Strong Representative Impossibility]\label{theorem:strong_rep_impossibility_appendix}
	For any $(\mathcal I, \mathcal D_{\mathcal I}, \mathcal C)$, when there is at least one cyclically privileged issue,\footnote{\emph{i.e.}, issues whose privilege graphs contain simple cycles of length \emph{at least $3$} (Definition \ref{def:privilege_graph}).} no representational mechanism simultaneously satisfies PPE, S-PIIA, and S-PC for all $\mathcal D_{\mathcal P}$.

	Meanwhile, given any mapping $\phi$ from each issue $i$ to a privilege graph $\phi(i)$ without simple cycles of length at least $3$, there exist a candidate space $\mathcal C$ whose privilege graph $G_i$ (for each $i$) is $\phi(i)$ or its supergraph, and a representational mechanism $f$ over $(\mathcal I, \mathcal D_{\mathcal I}, \mathcal C)$ satisfying PPE, S-PIIA, and S-PC for all $\mathcal D_{\mathcal P}$.
\end{theorem}

Theorem \ref{theorem:strong_rep_impossibility_appendix} is a generalization of Arrow's theorem and Theorem \ref{theorem:weak_rep_impossibility_appendix} to arbitrary candidate spaces $\mathcal C$. It shows that the cyclicity is both necessary and sufficient for the impossibility of a representative mechanism satisfying the axioms. However, the ``necessary'' part is not as strong as one would hope, since it constructs counterexample for at least one $\mathcal{C}$ associated with each non-cyclic privilege graph, instead of constructing counterexamples for all $\mathcal{C}$ associated with the privilege graph. It is an open question whether such a stronger necessity holds, and shall be the subject of future research.

The $N\geq 3$ condition (``at least 3 vertices in the $N$-clique'')\footnote{When $\mathcal C = \mathrm{LO}(N)^{\mathcal I}$, as is assumed by Arrow's theorem and Thm. \ref{theorem:weak_rep_impossibility_appendix}, the $N$ outcomes form a clique in the privilege graph.} in Theorem \ref{theorem:weak_rep_impossibility_appendix} and the original Arrow's theorem is replaced by the cyclicity condition here (``at least 3 vertices in a cycle''). Intuitively, the latter is a more precise condition that identifies a necessary-and-sufficient ``minimal structure'' in the candidate space $\mathcal C$ leading to impossibility. It can be checked in linear time (in the number of vertices and edges) for any privilege graph, since it's equivalent to the existence of strongly connected components of size at least $3$.

\section{Conclusion}

In this paper, we have formulated the problem of representative social choice, where a mechanism aggregates the preferences of a population based on a finite sample of individual-issue pairs. We have derived results that reflect both optimistic and pessimistic aspects of representative social choice. 

\paragraph{Implications for AI Alignment} % Be specific!
Representative social choice can model the alignment of AI systems to diverse human preferences. Generalization analysis of social choice mechanisms naturally apply to alignment mechanisms, while impossibility results highlight trade-offs between fairness and utility in alignment. These insights can guide the development of robust alignment strategies that manage these trade-offs explicitly.

\paragraph{Limitations and Future Directions} We focused on the generalization properties of representational mechanisms without studying other important properties, such as incentive compatibility and computational tractability. Future research could explore them and their interactions with generalization. \changed{Furthermore, our model does not account for noise or inconsistency in human preferences, a significant practical challenge. Other open directions include methods for learning the scoring function $s$ itself (Theorem \ref{theorem:gen_bound_scoring_appendix}) from data, or extending the model to handle out-of-distribution generalization to unseen outcomes or even unseen issues, which is vital for generative AI models.}

\paragraph{Broader Impact Statement} We aim to advance our understanding of democratic representation in collective decisions, with anticipated positive impacts. It helps develop more robust alignment strategies for AI systems, contributing to the equitable alignment of AI systems with diverse human preferences.

\acks{Many thanks to Anand Siththaranjan, Yifeng Ding, Yaowen Ye, Xiaoyuan Zhu, and Karim Abdel Sadek for helpful discussions.}

\newpage
\bibliography{sample}
\bibliographystyle{apalike}

\newpage
\appendix

\section{Additional Definitions}\label{sec:additional_definitions_appendix}

We have omitted the definition of the Vapnik-Chervonenkis dimension in the main text, given that it is a standard concept in statistical learning theory. We provide it here for completeness, translating it to our language of representative social choice.

\begin{definition}[Vapnik-Chervonenkis Dimension \cite{vapnik1999overview}]\label{def:vc_dimension_appendix}
	Given any issue space $\mathcal{I}$, we consider candidate profiles mapping $\mathcal{I}$ to $\mathrm{LO}(2)$, and a candidate space $\mathcal{C} \subseteq {\mathrm{LO}(2)}^{\mathcal{I}}$. The VC dimension of $\mathcal{C}$ is the cardinality of the largest finite set of issues $I \subseteq \mathcal{I}$ such that for any binary function $o\in{\mathrm{LO}(2)}^I$, there exists a preference profile $c \in \mathcal C$ such that $c(i) = o(i)$ for all $i \in I$. If $I$ can be arbitrarily large, then the VC dimension is infinite. Here we assume that the VC dimension is nonzero.
\end{definition}

\changed{\begin{example}[Single-Peaked Preferences]\label{ex:single_peaked}
Consider a case where issues $i \in \mathcal{I} = \mathbb{R}$ are points on a political spectrum, and outcomes are $N=2$ (`approve', `disapprove'). A candidate $C \in \mathcal{C}$ might be defined by an `interval' $[p_C - w_C, p_C + w_C]$, such that they 'approve' of issue $i$ if and only if $i$ falls within their interval. The class of such `interval' functions has a finite VC dimension of $2$, even though $\mathcal{I}$ is infinite. This models candidates with coherent, single-peaked political platforms.
\end{example}}

\section{Proofs in Binary Representative Social Choice}\label{sec:proofs_appendix_binary}

In this appendix, we provide the proofs of the results in Section \ref{sec:binary_representative}.

\subsection{Proof of Theorem \ref{theorem:bin_gen_bound_appendix}}

\begin{proof}
	\changed{We adapt the standard VC generalization bound \cite{vapnik1999overview}, which states that for a class $\mathcal{C}$ of binary functions with finite $\mathrm{VC}(\mathcal C)$, for any $\epsilon, \delta > 0$, the uniform convergence bound $\mathrm{Pr}[\sup_{C \in \mathcal C} |\hat R_S(C) - R(C)| \le \epsilon] \ge 1-\delta$ holds\footnote{\ldots where $\hat R_S(C)$ gives the empirical accuracy of function $C$ on samples $S$ and $R(C)$ gives the expected accuracy of $C$ on a random sample.} if the sample size $|\mathcal S|$ is polynomial in $(1/\epsilon)$, $\log(1/\delta)$, and $\mathrm{VC}(\mathcal C)$ (as stated in the theorem). Our proof reduces the representative social choice problem to this standard learning setting.}

	First consider the case where population has size $1$, and the resulting population $\mathcal{M}$ always maps an issue to a one-point distribution (\emph{i.e.}, a deterministic preference).

	In this case, $\mathcal M(i)_{C(i)}\in\{0,1\}$ represents whether the aggregated profile $C$ is correct on issue $i$. The population utility $\mathrm{E}_{i\sim {\mathcal D}_{\mathcal I}}[\mathcal M(i)_{C(i)}]$ can thus be formulated as the \emph{population error} in statistical learning settings, and the result in Theorem \ref{theorem:bin_gen_bound_appendix} follows directly from the VC generalization error bound in statistical learning theory \cite{vapnik1999overview}.

	To generalize the result to the case where the population has arbitrary cardinality, we can construct the following reduction to the deterministic case. We make the following definitions:
	\begin{align}
		\tilde{\mathcal I} &\coloneq \mathcal I \times \mathrm{LO}(2) \\
		{\tilde{\mathcal D}}_{\tilde{\mathcal I}}(i,b) &\coloneq \mathcal D_{\mathcal I}(i) \cdot {\mathcal M}(i)_b \\
		\tilde{\mathcal M}(i,b)_{b'} &\coloneq \mathbf 1_{b=b'} \\
		\tilde{\mathcal C} &\coloneq \{\tilde C: (i,b)\mapsto C(i) \  \mid \  C\in\mathcal C\} \\
		\tilde{\mathcal S} &\coloneq \{(o_1, (i_1, o_1)), (o_2, (i_2, o_2)), \ldots, (o_{|\mathcal S|}, (i_{|\mathcal S|}, o_{|\mathcal S|}))\}
	\end{align}
	where $b,b'\in\mathrm{LO}(2)$.

	In other words, we duplicate each issue $i$ into two issues $(i,0)$ and $(i,1)$, and for each issue $i$, we construct a population $\tilde{\mathcal M}(i,b)$ that always maps the issue to a one-point distribution, where the point is $b$. We then construct a saliency distribution $\tilde{\mathcal D}_{\tilde{\mathcal I}}$ that accounts for the original population probabilities over preferences. It can be verified that
	\begin{align}
		\mathrm{E}_{(i,b)\sim {\tilde{\mathcal D}}_{\tilde{\mathcal I}}}[\tilde{\mathcal M}(i,b)_{\tilde C(i,b)}] &= \mathrm{E}_{i\sim {\mathcal D}_{\mathcal I}}[\mathcal M(i)_{C(i)}] \\
		\frac 1{|\tilde{\mathcal S}|} \sum_{k=1}^{|\tilde{\mathcal S}|} \mathbf{1}_{{\tilde C}((i_k,o_k))={o_k}} &= \frac 1{|\mathcal S|} \sum_{k=1}^{|\mathcal S|} \mathbf{1}_{C(i_k)={o_k}} \\
		\mathrm{VC}(\tilde{\mathcal C}) &= \mathrm{VC}(\mathcal C)
	\end{align}

	Since Theorem \ref{theorem:bin_gen_bound_appendix} holds for the deterministic case $(\tilde{\mathcal I}, \tilde{\mathcal D}_{\tilde{\mathcal I}}, \tilde{\mathcal M}, \tilde{\mathcal C})$, it also holds for the original case $(\mathcal I, \mathcal D_{\mathcal I}, \mathcal M, \mathcal C)$.
\end{proof}

\subsection{Proof of Corollary \ref{cor:majority_vote_appendix}}

\begin{proof}
	From Theorem \ref{theorem:bin_gen_bound_appendix}, with probability $1-\delta$, we have
	\begin{align}
		U_{\mathcal{D}_{\mathcal I},\mathcal{M}}(f_{\mathrm{maj}}(\mathcal S))
		&\geq
		-\epsilon + \frac 1{|\mathcal S|} \sum_{k=1}^{|\mathcal S|} \mathbf{1}_{f_{\mathrm{maj}}(\mathcal S)(i_k)={o_k}} \\
		&=
		\max_{C\in\mathcal C}\left(-\epsilon + \frac 1{|\mathcal S|} \sum_{k=1}^{|\mathcal S|} \mathbf{1}_{f_{\mathrm{maj}}C(i_k)={o_k}}\right) \\
		&\geq
		-2\epsilon + \max_{C \in \mathcal C} U_{\mathcal{D}_{\mathcal I},\mathcal{M}}(C)
	\end{align}
\end{proof}

\section{Proofs in General Representative Social Choice}\label{sec:proofs_appendix_general}

In this appendix, we provide the proofs of the results in Section \ref{sec:general_representative}.

\subsection{Proof of Corollary \ref{cor:scoring_mechanism_appendix}}

\begin{proof}
	Proof of the corollary is done by bounding the empirical Rademacher complexity of the function class $\bar{\mathcal C}$ using Massart's Lemma \cite{bousquet2003introduction}, which gives us
    $$
        {\mathrm{\hat{R}}}_{|\mathcal S|}(\bar{\mathcal C}) \leq  c\sqrt{\frac{\log |\bar{\mathcal C}|}{|\mathcal S|}}
    $$
    for some constant $c > 0$.

    Given that the cardinality of $\bar{\mathcal C}$ is $(N!)^{|\mathcal I|}$, the Corollary is not hard to verify by plugging the bound into Theorem \ref{theorem:gen_bound_scoring_appendix}.
\end{proof}

\subsection{Proof of Lemma \ref{lem:pnd_all_mechanisms_appendix}}

\begin{proof}
	Assume otherwise, that under some mechanism $f$, a subpopulation ${\mathcal{D}'}_\mathcal{P}$ occupying a probability mass strictly less than $0.5$ can dictate the aggregation result in both directions, by being unanimous on either $c'\succ c$ or $c\succ c'$.

	Take another subpopulation ${\mathcal{D}''}_\mathcal{P}$ that's disjoint with ${\mathcal{D}'}_\mathcal{P}$ and has the same probability mass. Let $\mathcal{D'}_{\mathcal{P}}$ be unanimous on $c'\succ c$, and $\mathcal{D''}_{\mathcal{P}}$ be unanimous on $c\succ c'$. Since populations of the same probability mass are undisguishable to the mechanism, ${\mathcal{D}''}_\mathcal{P}$ must also dictate the aggregation result in both directions, leading to a contradiction.
\end{proof}

\subsection{Proof of Theorem \ref{theorem:weak_rep_impossibility_appendix}}

\begin{proof}
	The proof proceeds by a simple reduction to Arrow's impossibility theorem.
	Given $N\geq 3$, \changed{any instance of the classical Arrow problem with an odd number $n$ of voters can be simulated. We divide the population $\mathcal{D}_{\mathcal P}$ into $n$ disjoint subpopulations, each with mass $1/n$, and set each to be unanimous on one voter's preference profile.}
	
	\changed{Since $n$ is odd, there are no ties. The $1-e^{\Omega(|\mathcal S|)}$ convergence from W-PC and W-PIIA implies that, as $|\mathcal S| \to \infty$, the probability of our mechanism \emph{disagreeing} with the deterministic outcome dictated by the population majority goes to 0. This allows us to apply the logic of the deterministic Arrow's theorem to the limiting behavior of our mechanism.}
	
	\changed{The independence between issues (from $\mathcal C = \mathrm{LO}(N)^{\mathcal I}$) allows us to handle the multi-issue setting by examining each issue in isolation. The reduction is thus complete.}

	Since Arrow's theorem shows that for any problem instance with $N\geq 3$, no deterministic social choice mechanism can satisfy Pareto efficiency, independence of irrelevant alternatives, and non-dictatorship simultaneously, we are able to reduce at least one such hard instance to every instance $(\mathcal I, \mathcal D_{\mathcal I}, \mathcal C = \mathrm{LO}(N)^{\mathcal I})$ of the weak representative setting, and the theorem follows.
\end{proof}

\subsection{Proof of Lemma \ref{lem:transitivity_privilege_appendix}}

\begin{proof} 
	For any extension $C\in \mathrm{LO}(N)^{\mathcal I}$ of $u\succ w$, we need to show that $C'\coloneqq C\odot_i \sigma_{(u,w)}\in\mathcal{C} \implies C\in\mathcal{C}$. Assuming $C'\in\mathcal{C}$, we consider ordering between $u,v,w$ in $C(i)$.

	When $u\succ_C w\succ_C v$, we have $v\succ_{C'} w\succ_{C'} u$. Since $u\succ v$ is privileged, $C''\coloneqq C'\odot_i \sigma_{(u,v)}\in\mathcal C$, and we have $u\succ_{C''} w\succ_{C''} v$. Since $v\succ w$ is privileged, $C'''\coloneqq C''\odot_i \sigma_{(v,w)}\in\mathcal C$. It can be verified that $u\succ_{C'''} v\succ_{C'''} w$ and $C'''=C$.
	The case $v \succ_C u\succ_C w$ is analogous.

	When $u\succ_C v\succ_C w$, we have $w\succ_{C'} v\succ_{C'} u$. Since $v\succ w$ is privileged, $C''\coloneqq C'\odot_i \sigma_{(v,w)}\in\mathcal C$. Now $v\succ_{C''} w\succ_{C''} u$, reducing to the first case.
\end{proof}

\subsection{Proof of Lemma \ref{lem:closure_privilege_appendix}}

\begin{proof}
	For any extension $C\in \mathrm{LO}(N)^{\mathcal I}$ of $o\oplus o'$ and permutation $\sigma\in\mathfrak{S}_{\{c_1,c_2,\cdots,c_{k+m-1}\}}$, we need to show that $C'\coloneqq C\odot_i \sigma\in\mathcal{C} \implies C\in\mathcal{C}$.

	Similar to the proof of Lemma \ref{lem:transitivity_privilege_appendix}, there are two transformations that we are allowed to apply to $C'$ to obtain $C$: first, we can apply some permutation $\sigma_o\in \mathfrak{S}_{\{c_1,c_2,\cdots,c_k\}}$ on $C$ to sort $\{c_1,c_2,\cdots,c_k\}$ in this order if it isn't already sorted; second, we can apply some permutation $\sigma_{o'}\in \mathfrak{S}_{\{c_k,c_{k+1},\cdots,c_{k+m-1}\}}$ on $C$ to sort $\{c_k,c_{k+1},\cdots,c_{k+m-1}\}$ in this order if it isn't already. Since $o$ and $o'$ are privileged, these two transformations preserve membership in $\mathcal{C}$.

	We repeated apply these two transformations in arbitrary order, until neither of them can be applied --- \emph{i.e.}, until both $\{c_1,c_2,\cdots,c_k\}$ and $\{c_k,c_{k+1},\cdots,c_{k+m-1}\}$ are sorted in the correct order. Since $\left( c_1\succ c_2\succ\cdots\succ c_k \right) \land \left( c_k\succ c_{k+1}\succ c_{k+2}\succ\cdots\succ c_{k+m-1} \right) \implies c_1\succ c_2\succ\cdots\succ c_{k+m-1}$, when such a process terminates, we have $C\in\mathcal{C}$.

	We still need to show that the process indeed terminates. We define a \emph{potential function} $\Phi$ such that for any $E\in \mathrm{LO}(N)^{\mathcal I}$, $\Phi(E)=\mathrm{Inv}_{c_1 \succ \cdots\succ c_{k+m-1}}(E(i))=\sum_{1\leq x<y \leq k+m-1}\mathbf{1}_{c_y \succ_{E(i)} c_x}$, the number of inversions in $E(i)$ with respect to the ordering $o\oplus o': c_1 \succ \cdots\succ c_{k+m-1}$.

	We show that the transformation $\odot_i \sigma_o$ strictly decreases $\mathrm{Inv}_{o\oplus o'}$. We first decompose $\odot_i \sigma_o$ into a series of transpositions $\odot_i\sigma_{(c_{j_1},c_{j_2})}$ such that $j_1<j_2$ and $c_{j_2}\succ c_{j_1}$ in the pre-transposition ordering. It can be verified that this transposition removes the inversion $(c_{j_1},c_{j_2})$ without introducing new inversions.
	Therefore the transformation $\odot_i \sigma_o$ strictly decreases $\Phi$, and likewise for $\odot_i \sigma_{o'}$. Since the initial $\Phi(C')$ is finite and $\Phi$ is always non-negative, the process must terminate.
\end{proof}

\subsection{The Field Expansion Lemma}

The Field Expansion Lemma is a key instrument in proving the strong representative impossibility theorem. It shows that the decisiveness of a subpopulation over a specific preference can be expanded to decisiveness over all pairwise orderings among a set of outcomes. We state and prove the lemma in this section.

\begin{definition}[Decisiveness]
	Given any representational mechanism $f$, for a subpopulation ${\mathcal{D}'}_\mathcal{P}$ with probability mass $0<|{\mathcal{D}'}_\mathcal{P}|\leq 1$ in the larger population $\mathcal{D}_{\mathcal P}$, we say that it is \emph{decisive over $c\succ c'$} in issue $i$ if when ${\mathcal{D}'}_\mathcal{P}$ is unanimous on $c\succ c'$, whatever the preference specification $\mathcal{D}_\mathcal{P}$ of the whole population is, $c\succ_{f(\mathcal S)} c'$ with probability $1-e^{\Omega(|\mathcal S|)}$. 
	We say that it is \emph{weakly decisive} if when ${\mathcal{D}'}_\mathcal{P}$ is unanimous on $c\succ c'$ and the rest of the population is unanimous on $c'\succ c$, $c\succ_{f(\mathcal S)} c'$ with probability $1-e^{\Omega(|\mathcal S|)}$.
\end{definition}

\begin{remark}
	Given the guaranteed anonymity of our problem setting, all subpopulations of the same probability mass are entirely interchangeable, and thus the decisiveness of a subpopulation over a specific preference only depends on the subpopulation's mass.
\end{remark}

\begin{lemma}[Field Expansion Lemma]\label{lem:field_expansion_appendix}
	Given any representational mechanism $f$ satisfying PPE and S-PIIA, consider outcomes $u,v,w$ of issue $i\in\mathcal{I}$ such that all $6$ pairwise orderings among them are privileged, and a subpopulation ${\mathcal{D}'}_\mathcal{P}$ that is weakly decisive over $u\succ v$. We then have that ${\mathcal{D}'}_\mathcal{P}$ is decisive over $u\succ w$. 
	
	Note that by analogy, it will also be decisive over $w\succ v$. With repeated application and transitivity, it follows that ${\mathcal{D}'}_\mathcal{P}$ is decisive over all $6$ pairwise orderings among $u,v,w$.
	% If $v \succ w$ is privileged in issue $i$, then ${\mathcal{D}'}_\mathcal{P}$ is decisive over $u\succ w$ in issue $i$. Similarly, if $w \succ u$ is privileged in issue $i$, then ${\mathcal{D}'}_\mathcal{P}$ is decisive over $w\succ v$ in issue $i$.
\end{lemma}
\begin{proof}
	Assume that ${\mathcal{D}'}_\mathcal{P}$ is unanimous on $u\succ w$. By S-PIIA, we can assume that ${\mathcal{D}'}_\mathcal{P}$ is further unanimous on $u\succ v\succ w$, while the rest of the population is unanimous on $v\succ u$ and $v\succ w$. This does not affect $f(\mathcal S)\mid_{\mathrm{LO}(\{u,w\})}$.
	
	By weak decisiveness, we have $u\succ_{f(\mathcal S)} v$ with probability $1-e^{\Omega(|\mathcal S|)}$. Due to the privilegedness of $v\succ w$, PPE can be applied to every extension of the opposite ordering $w\succ v$ (where the extension serve as $C'$), guaranteeing that $v\succ_{f(\mathcal S)} w$ with probability $1-e^{\Omega(|\mathcal S|)}$. Taken together, we have $u\succ_{f(\mathcal S)} w$ with probability $1-e^{\Omega(|\mathcal S|)}$, and thus ${\mathcal{D}'}_\mathcal{P}$ is decisive over $u\succ w$. The rest follows.
\end{proof}

Finally, we review a few graph theoretical concepts that will prove useful in the proof of the strong representative impossibility theorem.

\begin{definition}[Strongly Connected Component]
	For a directed graph $G$, a \emph{strongly connected component} (SCC) is a maximal subgraph of $G$ (including those of size $1$) in which there is a directed path between every ordered pair of vertices. The SCCs of $G$ form a partition of the vertex set of $G$.
\end{definition}

\begin{definition}[Condensation]
	The \emph{condensation} $\mathrm{cond}(G)$ of a directed graph $G$ is a directed acyclic graph (DAG) where each vertex represents an SCC of $G$, and there is an edge from SCC $A$ to SCC $B$ iff there is an edge from a vertex in $A$ to a vertex in $B$ in $G$.
\end{definition}

\begin{remark}
	Intuitively, SCC is the unit of bidirectional connectivity in a directed graph, and condensation is a way to simplify a directed graph into a DAG by contracting SCCs into single vertices. There is a directed path from $u$ to $v$ in $G$ iff there is a directed path from the SCC containing $u$ to the SCC containing $v$ in $\mathrm{cond}(G)$. Both SCC and condensation can be computed in linear time in the number of vertices and edges of the graph, using algorithms such as Tarjan's algorithm and Kosaraju's algorithm.
  
  Note that given the transitivity of the privilege graph, any SCC in the privilege graph is a complete digraph.
\end{remark}

\subsection{Proof of Theorem \ref{theorem:strong_rep_impossibility_appendix}}

\begin{proof}
	We first show that issues with cyclic privilege graphs imply the impossibility of a mechanism satisfying the axioms. We assume the existence of such an $f$ and show that it leads to a contradiction.
	
	For any cyclic privileged issue $i$, by Lemma \ref{lem:transitivity_privilege_appendix}, we know that any directed simple cycle in $G_i$ induces a complete digraph on the outcomes in the cycle. We arbitrarily pick outcomes $u,v,w$ from the cycle. Arbitrarily partition the population into equal portions ${\mathcal{D}^\mathrm{a}}_\mathcal{P}$ and ${\mathcal{D}^\mathrm{b}}_\mathcal{P}$ of probablity mass $\frac 23$ and $\frac 13$ respectively. Let ${\mathcal{D}^\mathrm{a}}_\mathcal{P}$ be unanimous on $u\succ v$ and ${\mathcal{D}^\mathrm{b}}_\mathcal{P}$ be unanimous on $v\succ u$. By S-PC, the representational mechanism $f$ must converge upon $u\succ v$ or $v\succ u$ with probability $1-e^{\Omega(|\mathcal S|)}$ in the full population. 

	Assume that $f$ converges upon $u\succ v$, in which case ${\mathcal{D}^\mathrm{a}}_\mathcal{P}$ is weakly decisive over $u\succ v$. By the field expansion lemma, ${\mathcal{D}^\mathrm{a}}_\mathcal{P}$ is decisive over $u\succ v$. We further partition ${\mathcal{D}^\mathrm{a}}_\mathcal{P}$ into $1:2$ portions ${\mathcal{D}^\mathrm{a,1}}_\mathcal{P}$ and ${\mathcal{D}^\mathrm{a,2}}_\mathcal{P}$. Then, we let:
	\vspace{-0.5em}
	\begin{itemize}[leftmargin=2em]
		\item ${\mathcal{D}^\mathrm{a,1}}_\mathcal{P}$ (probability mass $\frac 29$) be unanimous on $u\succ v\succ w$,
		\item ${\mathcal{D}^\mathrm{a,2}}_\mathcal{P}$ (probability mass $\frac 49$) be unanimous on $w\succ u\succ v$, and
		\item ${\mathcal{D}^\mathrm{b}}_\mathcal{P}\ \;$ (probability mass $\frac 13$) be unanimous on $v\succ w\succ u$.
	\end{itemize}
	\vspace{-0.5em}
	Since ${\mathcal{D}^\mathrm{a}}_\mathcal{P}$ is decisive, we have $u \succ_{f(\mathcal S)} v$ with probability $1-e^{\Omega(|\mathcal S|)}$. Since the probability masses of all subpopulations have denominators $3$ or $9$ (both odd numbers), there can be no ties, and $f$ must converge due to S-PC. Given that $u \succ_{f(\mathcal S)} v$, with probability $1-e^{\Omega(|\mathcal S|)}$, we have either $u \succ_{f(\mathcal S)} w$ or $w \succ_{f(\mathcal S)} v$. In the former case, ${\mathcal{D}^\mathrm{a,1}}_\mathcal{P}$ is weakly decisive over $u\succ w$, and in the latter case, ${\mathcal{D}^\mathrm{a,2}}_\mathcal{P}$ is weakly decisive over $w\succ v$. In either case, the subpopulation is decisive by the field expansion lemma, but its probability mass ($\frac 29$ or $\frac 49$) is smaller than $0.5$, contradicting PND (Lemma \ref{lem:pnd_all_mechanisms_appendix}).

	We can similarly show that $f$ converging upon $v\succ u$ leads to a contradiction. Therefore, no such $f$ can exist. Note that in the proof above, we have been misusing $|\mathcal S|$ to denote the number of samples that falls into issue $i$, which is not a problem since $i$'s probability mass $\mathrm{Pr}_{\mathcal D_{\mathcal I}}[i]>0$ is a positive constant (we can pick the $i$ with the largest probability mass, which makes the mass only dependent on $\mathcal D_{\mathcal I}$ itself), and the asymptotic behavior of the convergence probability is the same.

	Then, assuming that all issues (except possibly those with zero probability in $\mathcal D_{\mathcal I}$) have specific non-cyclic privilege graphs, we construct a candidate space $\mathcal C$ consistent with the privilege graphs, and a mechanism $f$ that satisfies the axioms for all $\mathcal D_{\mathcal P}$.

	First consider the case where $|\mathcal I|=1, {\mathcal I}=\{i\}$. Since $\phi(i)$ is transitive (Lemma \ref{lem:transitivity_privilege_appendix}) but doesn't have simple cycles of length at least $3$, it follows that $\phi(i)$ doesn't have any SCC containing 3 or more vertices. Fix an arbitrary topological order $(g_1,\cdots,g_m)$ of vertices in the DAG $\mathrm{cond}(\phi(i))$, where $\{g_j\}_{j=1}^m$ constitute a partition of $[N]$, and $|\{g_j\}|\in\{1,2\}$ for each $j$; let $g_{k_1}, \cdots, g_{k_l}$ be the SCCs of size $2$. Now consider translating $(g_1,\cdots,g_m)$ into an ordering in $\mathrm{LO}(N)$, where the vertices in each SCC are ordered arbitrarily, and vertices in different SCCs are ordered according to the topological order $(g_1,\cdots,g_m)$. There are $2^l$ ways to perform such translation, and we define $\mathcal C\subset \mathrm{LO}(N)$ to be the set of these $2^l$ orderings. It can be verified that the $G_i$ resulting from $\mathcal C$ contains $\phi(i)$ as a subgraph.

	We then define the mechanism $f$. For each $g_{k_j} = \{u_j,v_j\}\ (1\leq j\leq l)$, it examines if a majority in the sampels prefer $u_j$ over $v_j$. If yes, it outputs the ordering that places $u_j$ above $v_j$; otherwise, it outputs the ordering that places $v_j$ above $u_j$. Ties are broken arbitrarily. Let us now verify that $f$ satisfies PPE, S-PIIA, and S-PC for all $\mathcal D_{\mathcal P}$.
	\begin{itemize}[leftmargin=2em]
		\item PPE: Relative orderings between outcomes in different SCCs are fixed by the topological order, and therefore the outcomes $c,c'$ in the PPE axiom must be in the same SCC. $f$ resolves this preference using a majority vote, which, by Hoeffding's inequality, has a convergence probability of $1-e^{\Omega(|\mathcal S|)}$.
		\item S-PIIA: It can be verified that if $c\succ c'$ and $c'\succ c$ are both privileged orderings in $i$, then $c,c'$ must be in the same SCC. $f$ resolves this preference using a majority vote that considers only the preference between $c$ and $c'$, and thus satisfies S-PIIA.
		\item S-PC: Again, $c,c'$ must be in the same SCC. $f$ resolves this preference using a majority vote, which, by Hoeffding's inequality, has a convergence probability of $1-e^{\Omega(|\mathcal S|)}$ when the population is non-uniform on $c,c'$.
	\end{itemize}
	
	In the general case where $|\mathcal I|>1$, we can apply the above construction to each issue independently, and define $\mathcal C$ to be the Cartesian product of the constructed candidate spaces. The mechanism $f$ is then defined to apply the constructed mechanism for each issue independently. It can be verified that $f$ satisfies PPE, S-PIIA, and S-PC for all $\mathcal D_{\mathcal P}$.
\end{proof}
\end{document}